\documentclass[12pt]{article}
\usepackage{amsmath, amssymb, amsfonts}
\usepackage{amsthm}
\usepackage{geometry}
\usepackage{cite}
\usepackage[colorlinks=true,linkcolor=blue,citecolor=blue,urlcolor=blue]{hyperref}

\numberwithin{equation}{section}

\newtheorem{theorem}{Theorem}[section]

\newtheorem{proposition}{Proposition}[section]
\newtheorem{corollary}{Corollary}[section]

\theoremstyle{definition}
\newtheorem{definition}{Definition}[section]

\title{Necessary and sufficient conditions for universality of Kolmogorov--Arnold networks}

\author{Vugar E. Ismailov\thanks{E-mail: \texttt{vugaris@gmail.com}, \texttt{vugaris@mail.ru}.}}

\date{}

\begin{document}

\maketitle

\begin{abstract}
We analyze the universal approximation property of
Kolmogorov--Arnold Networks (KANs) in terms of their edge functions.
If these functions are all affine, then universality clearly fails. 
How many non-affine functions are needed, in addition to affine ones, 
to ensure universality? We show that a single one suffices.
More precisely, we prove that deep KANs in which all edge functions are
either affine or equal to a fixed continuous function $\sigma$ are dense
in $C(K)$ for every compact set $K\subset\mathbb{R}^n$ if and only if
$\sigma$ is non-affine. In contrast, for KANs with exactly two hidden
layers, universality holds if and only if $\sigma$ is nonpolynomial.
We further show that the full class of affine functions is not required;
it can be replaced by a finite set without affecting universality. 
In particular, in the nonpolynomial case, a fixed family of five 
affine functions suffices when the depth is arbitrary.
More generally, for every continuous non-affine
function $\sigma$, there exists a finite affine family $A_\sigma$ such
that deep KANs with edge functions in $A_\sigma\cup\{\sigma\}$ remain
universal. We also prove that KANs with the spline-based edge parameterization
introduced by Liu et al.~\cite{Liu2024} are universal approximators in the
classical sense, even when the spline degree and knot sequence are fixed
in advance.


\end{abstract}

\section{Introduction}

Kolmogorov--Arnold Networks (KANs) are a neural network architecture
recently introduced in the machine learning literature \cite{Liu2024}.
In contrast to classical multilayer perceptrons (MLPs), where a nonlinear
activation function is applied at each node and each connection is assigned
a scalar weight, KANs follow a different design: univariate functions,
referred to here as edge functions, are assigned to edges, while nodes
perform only summation. This leads naturally to representations of
multivariate functions as compositions of univariate functions combined
by addition.

From a mathematical point of view, this construction is closely related to the 
Kolmo\-gorov--Arnold representation theorem. In its classical form, 
this theorem asserts that every continuous function
\[
f:[0,1]^n \to \mathbb{R}
\]
admits a representation of the form
\begin{equation}\label{eq:kart_intro}
f(x_1,\dots,x_n)
=
\sum_{k=1}^{2n+1}
\Phi_k\!\left(
\sum_{j=1}^{n}\psi_{kj}(x_j)
\right),
\end{equation}
where $\psi_{kj}:\mathbb{R}\to\mathbb{R}$ are fixed continuous functions 
and $\Phi_k:\mathbb{R}\to\mathbb{R}$ are continuous functions depending 
on $f$ (see \cite{Kolmogorov1957,Arnold1959}). This theorem has been revisited, 
refined, and extended in many directions (see, e.g., \cite[Chapter~1]{Khavinson1997}, 
\cite[Chapter~4]{Ismailov2021} and the references therein). 
It should be noted that representation \eqref{eq:kart_intro} remains
valid not only for continuous functions, but also for all discontinuous
functions $f$, with the same inner functions $\psi_{kj}$ as in the
continuous case and with outer functions $\Phi_k$ depending on $f$
(see \cite{Ismailov2023}). Moreover, the inner functions $\psi_{kj}(x_j)$ can
be chosen in the form $\lambda_j \psi(x_j+\epsilon k)$, and the outer
functions $\Phi_k$ can be replaced by a single function $\Phi$, both in
the continuous case \cite{Lorentz1962,Sprecher1965} and in the
discontinuous case \cite{Ismayilova2024}.

The structure of \eqref{eq:kart_intro} is important: it involves exactly
two stages of summation and compositions of univariate functions.
First, each inner sum $\sum_{j=1}^n \psi_{kj}(x_j)$ combines univariate
functions of the independent variables. Then an outer summation aggregates
the resulting expressions through the functions $\Phi_k$. In this way, the
Kolmogorov--Arnold representation can be viewed as a two-stage construction
built from univariate functions and addition.

Kolmogorov--Arnold Networks can be understood as an iteration of this construction. 
While the classical theorem involves two summation stages, 
KANs allow several such stages arranged in layers. Each layer applies sums of 
univariate functions to the outputs of the previous layer, producing a deeper 
composition. In this sense, KANs repeat the basic structure of \eqref{eq:kart_intro} 
across multiple layers, which explains the connection between the Kolmogorov--Arnold 
theorem and these networks.

We now give a definition of Kolmogorov--Arnold Networks, following \cite{Liu2024}.

\begin{definition}
Let $x=(x_1,\dots,x_{n_0}) \in \mathbb{R}^{n_0}$. A \emph{Kolmogorov--Arnold Network} 
with $L$ hidden layers is defined recursively by
\begin{align}
h^{(0)}_j(x) &= x_j, \qquad j=1,\dots,n_0, \label{eq:input_intro}\\
h^{(\ell)}_k(x) &= \sum_{j=1}^{n_{\ell-1}} \psi^{(\ell)}_{kj}\bigl(h^{(\ell-1)}_j(x)\bigr),
\qquad \ell=1,\dots,L,\quad k=1,\dots,n_\ell, \label{eq:hidden_intro}\\
y_i(x) &= \sum_{k=1}^{n_L} \phi_{ik}\bigl(h^{(L)}_k(x)\bigr),
\qquad i=1,\dots,m. \label{eq:output_intro}
\end{align}
Here $\psi^{(\ell)}_{kj}:\mathbb{R}\to\mathbb{R}$ and
$\phi_{ik}:\mathbb{R}\to\mathbb{R}$ are univariate functions,
called \emph{edge functions}.
\end{definition}

In this architecture, activation functions are associated with edges, whereas nodes 
perform only summation. Each $k$-th node of the $\ell$-th layer computes
\[
h^{(\ell)}_k = \sum_{j=1}^{n_{\ell-1}} \psi^{(\ell)}_{kj} \circ h^{(\ell-1)}_j,
\]
which combines additive structure with compositions of univariate functions. 
The network has $L$ hidden layers; including input and output layers, the total depth is $L+2$.

By recursively substituting \eqref{eq:hidden_intro} into \eqref{eq:output_intro}, 
one obtains the fully expanded nested representation
\begin{equation*}
\begin{aligned}
y_i(x)
&=
\sum_{k_L=1}^{n_L}
\phi_{ik_L}
\Biggl(
\sum_{k_{L-1}=1}^{n_{L-1}}
\psi^{(L)}_{k_L k_{L-1}}
\Biggl(
\sum_{k_{L-2}=1}^{n_{L-2}}
\psi^{(L-1)}_{k_{L-1} k_{L-2}}
\Biggl(
\cdots
\sum_{k_1=1}^{n_1}
\psi^{(2)}_{k_2 k_1}
\Bigl(
\sum_{j=1}^{n_0}
\psi^{(1)}_{k_1 j}(x_j)
\Bigr)
\cdots
\Biggr)
\Biggr)
\Biggr).
\end{aligned}
\end{equation*}
This formula shows that a KAN represents functions as nested compositions of 
sums of univariate functions, with each additional layer introducing a new 
level of composition while preserving the additive structure.

For clarity, we write explicitly the case $L=2$:
\begin{align*}
h^{(1)}_{k_1}(x)
&=
\sum_{j=1}^{n_0}
\psi^{(1)}_{k_1 j}(x_j),
\\
h^{(2)}_{k_2}(x)
&=
\sum_{k_1=1}^{n_1}
\psi^{(2)}_{k_2 k_1}\bigl(h^{(1)}_{k_1}(x)\bigr),
\\
y_i(x)
&=
\sum_{k_2=1}^{n_2}
\phi_{i k_2}\bigl(h^{(2)}_{k_2}(x)\bigr),
\end{align*}
so that
\[
y_i(x)
=
\sum_{k_2=1}^{n_2}
\phi_{i k_2}
\Biggl(
\sum_{k_1=1}^{n_1}
\psi^{(2)}_{k_2 k_1}
\Bigl(
\sum_{j=1}^{n_0}
\psi^{(1)}_{k_1 j}(x_j)
\Bigr)
\Biggr).
\]
The cases $L=3,4$, and so on, admit completely analogous expansions.

Each layer may also be viewed as an operator
\[
T^{(\ell)} : \mathbb{R}^{n_{\ell-1}} \to \mathbb{R}^{n_\ell},
\qquad
(T^{(\ell)}h)_k = \sum_{j=1}^{n_{\ell-1}} \psi^{(\ell)}_{kj}(h_j).
\]
The output layer defines an operator
\[
\Phi : \mathbb{R}^{n_L} \to \mathbb{R}^m,
\qquad
(\Phi h)_i = \sum_{k=1}^{n_L} \phi_{ik}(h_k),
\]
so that a deep KAN can be written concisely as
\[
y = \Phi \circ T^{(L)} \circ \cdots \circ T^{(1)} (x).
\]

It is important to note that when $L=1$, the KAN reduces to
\[
y_i(x)=\sum_{k=1}^{n_1}\phi_{ik}\!\left(\sum_{j=1}^{n_0}\psi_{kj}(x_j)\right),
\qquad i=1,\dots,m.
\]
In the scalar-output case $m=1$, this coincides exactly with the 
Kolmogorov--Arnold representation \eqref{eq:kart_intro}.

It should be noted that a model closely related to present-day KANs
appeared earlier in 2003, where a Kolmogorov-type
superposition is employed, with both inner and outer univariate
functions represented by adaptive cubic splines (see \cite{IgelnikParikh2003}). 
Other related interpretations, including connections with the Urysohn operator 
and its discrete form, which are viewed as generalized additive models (GAMs), 
have been discussed in \cite{PolarPoluektov2021,PoluektovPolar2025}. The latter 
works interpret the Kolmogorov--Arnold representation as a composition (tree) 
of discrete Urysohn operators, referred to as a Urysohn tree.

In recent years, there has been growing interest in the theoretical study of
Kolmo\-gorov--Arnold Networks. Various topics related to KANs and their variants,
including approximation rates
\cite{LiuChatziLai2025,KratsiosKimFuruya2025},
representation and specific universality results
\cite{Gleyzer2025,Chiu2026,Toscano2025},
expressivity \cite{Wang2024}, and learning-theoretic and optimization
aspects \cite{ZhangZhou2025,Eshtehardian2026}, have been investigated. 
For a general overview of Kolmogorov--Arnold Networks and their variants,
we refer the reader to \cite{Somvanshi2025,Noorizadegan2025} 
and the references therein.

While recent progress has been made in understanding the approximation
properties of KAN-type models, a systematic study of universality in terms
of edge functions is still lacking. The aim of this paper is to
address this question for KANs.

Our main result provides a complete characterization of universality for
KANs in which all edge functions are either affine or equal to a fixed
continuous function $\sigma$: deep KANs of this type are dense in $C(K)$
for every compact set $K\subset\mathbb{R}^n$ if and only if $\sigma$ is
non-affine (Theorem~\ref{thm:kan_iff}). In addition, we show that for
KANs with exactly two hidden layers, universality holds if and only if
$\sigma$ is nonpolynomial (Theorem~\ref{thm:kan_fixed_sigma_affine}).
These results clarify the conditions under which universality holds in
the deep and two-hidden-layer settings.

As a consequence, we also obtain a universality result for the class of edge 
functions used in the original KAN model \cite{Liu2024}. More precisely, 
we show that KANs whose edge functions are given by linear combinations of a 
fixed base function and spline functions are dense in $C(K)$ 
(Corollary~\ref{cor:kan_original_edge}). In this result, both the 
degree of the splines and the number of knot points can be fixed in advance.

We further prove that the full family of affine functions can be replaced 
by a finite one without affecting universality. In particular, in the 
nonpolynomial case, a fixed family of five affine functions 
suffices when the depth is arbitrary
(Theorem~\ref{thm:nonpoly_five_affine}). More generally, for every 
continuous non-affine function $\sigma$, there exists a finite affine 
family $A_\sigma$ such that deep KANs with edge functions in 
$A_\sigma\cup\{\sigma\}$ remain universal 
(Theorem~\ref{thm:finite_affine_nonaffine}).

\section{A simple universality result based on multilayer perceptrons}

This section presents a simple observation that serves as a starting point
for the subsequent sections. We show that KANs inherit the universal
approximation property from classical multilayer perceptrons by
identifying MLPs as a special case of KANs under a suitable restriction
on the edge functions.

We consider the single-hidden-layer KAN representation
\begin{equation}\label{eq:kan_uat}
F(x)
=
\sum_{k=1}^{N}
\phi_k\!\left(
\sum_{j=1}^{n}
\psi_{kj}(x_j)
\right),
\qquad x\in\mathbb{R}^n.
\end{equation}

\begin{proposition}\label{prop:mlp_in_kan}
Every function representable by a single-hidden-layer multilayer 
perceptron can be written in the form \eqref{eq:kan_uat}.
\end{proposition}

\begin{proof}
Let
\[
G(x)=\sum_{k=1}^{N} c_k\,\sigma(w_k\cdot x+b_k),
\]
be a standard single-hidden-layer MLP, where $\sigma:\mathbb{R}\to\mathbb{R}$ 
is an activation function, $w_k=(w_{k1},\dots,w_{kn})\in\mathbb{R}^n$, 
and $c_k,b_k\in\mathbb{R}$.

For each $k=1,\dots,N$, define edge functions
\[
\psi_{kj}(t)=w_{kj}t+\beta_{kj}, \qquad j=1,\dots,n,
\]
where the constants $\beta_{kj}$ are chosen so that
\[
\sum_{j=1}^n \beta_{kj}=b_k.
\]
Also define
\[
\phi_k(t)=c_k\,\sigma(t).
\]
Then
\[
\sum_{j=1}^{n}\psi_{kj}(x_j)
=
\sum_{j=1}^{n} w_{kj}x_j + b_k
=
w_k\cdot x+b_k,
\]
and therefore
\[
F(x)
=
\sum_{k=1}^{N}
\phi_k\!\left(
\sum_{j=1}^{n}
\psi_{kj}(x_j)
\right)
=
\sum_{k=1}^{N} c_k\,\sigma(w_k\cdot x+b_k)
=
G(x).
\]
\end{proof}

Proposition~\ref{prop:mlp_in_kan} shows that the class of
single-hidden-layer KANs contains all single-hidden-layer MLPs.
In fact, under the restriction that the inner edge functions are affine
and the outer ones are of the form $c_k\sigma(t)$, the two classes coincide.

We now deduce a simple universality result for KANs from the classical 
universal approximation theorem for MLPs.

\begin{theorem}\label{thm:kan_uat}
Let $\sigma:\mathbb{R}\to\mathbb{R}$ be continuous. 
Then the following are equivalent:
\begin{itemize}
\item[(i)] $\sigma$ is nonpolynomial;
\item[(ii)] for every compact set $K\subset\mathbb{R}^n$, 
the class of functions of the form
\[
F(x)
=
\sum_{k=1}^{N}
\phi_k\!\left(
\sum_{j=1}^{n}
\psi_{kj}(x_j)
\right),
\]
where each edge function $\psi_{kj}$ is affine and 
each edge function $\phi_k$ is of the form
\[
\phi_k(t)=c_k\,\sigma(t),
\]
is dense in $C(K)$.
\end{itemize}
\end{theorem}

\begin{proof}
By the construction given in the proof of Proposition~\ref{prop:mlp_in_kan},
the class under consideration coincides with the class of functions
representable by single-hidden-layer MLPs of the form
\[
x \mapsto \sum_{k=1}^{N} c_k\,\sigma(w_k\cdot x+b_k).
\]
It is a classical result (see \cite{Leshno1993}) that, for continuous
$\sigma:\mathbb{R}\to\mathbb{R}$, this class is dense in $C(K)$ for every
compact set $K\subset\mathbb{R}^n$ if and only if $\sigma$ 
is nonpolynomial. The conclusion follows.
\end{proof}

The above result extends immediately to vector-valued functions.

\begin{corollary}
Let $\sigma:\mathbb{R}\to\mathbb{R}$ be continuous and nonpolynomial, and let
$K\subset\mathbb{R}^n$ be compact. Then for every $f\in C(K;\mathbb{R}^m)$
and every $\varepsilon>0$, there exists a function
$F=(F_1,\dots,F_m):K\to\mathbb{R}^m$ of the form
\[
F_i(x)
=
\sum_{k=1}^{N}
\phi_{ik}\!\left(
\sum_{j=1}^{n}
\psi_{kj}(x_j)
\right),
\qquad i=1,\dots,m,
\]
where each $\psi_{kj}$ is affine and each $\phi_{ik}$ is given by
\[
\phi_{ik}(t)=c_{ik}\,\sigma(t),
\]
such that
\[
\sup_{x\in K}\|f(x)-F(x)\|_{\mathbb{R}^m}<\varepsilon.
\]
Here $\|\cdot\|_{\mathbb{R}^m}$ denotes the Euclidean norm on $\mathbb{R}^m$
(or any other fixed norm on $\mathbb{R}^m$).
\end{corollary}

\begin{proof}
Apply Theorem~\ref{thm:kan_uat} to each component of $f$.
Since all norms on $\mathbb{R}^m$ are equivalent, the conclusion follows.
\end{proof}

All results obtained in the subsequent sections also extend in a straightforward 
manner to vector-valued functions. For brevity, we formulate them only in the
scalar-valued setting.

\section{Universality with all affine and a fixed nonpolynomial edge function}

In this section, we prove a universality result for deep
Kolmogorov--Arnold Networks in which each edge function is either affine
or equal to a fixed continuous function $\sigma$. The proof is based on
an exact realization of classical single-hidden-layer multilayer
perceptrons by means of two-hidden-layer KANs.

The following theorem shows that universality holds if and only if
$\sigma$ is nonpolynomial.

\begin{theorem}\label{thm:kan_fixed_sigma_affine}
Let $\sigma:\mathbb{R}\to\mathbb{R}$ be continuous. Consider 
the class of KANs with two hidden layers
\begin{align}
h_k^{(1)}(x)
&=
\sum_{j=1}^{n}\psi_{kj}^{(1)}(x_j),
\qquad k=1,\dots,N, \label{eq:kan_layer1}\\
h_k^{(2)}(x)
&=
\sum_{r=1}^{N}\psi_{kr}^{(2)}\bigl(h_r^{(1)}(x)\bigr),
\qquad k=1,\dots,N, \label{eq:kan_layer2}\\
F(x)
&=
\sum_{k=1}^{N}\phi_k\bigl(h_k^{(2)}(x)\bigr), \label{eq:kan_output}
\end{align}
where each edge function $\psi_{kj}^{(1)}$, $\psi_{kr}^{(2)}$, 
and $\phi_k$ is either affine or equal to $\sigma$.

Then this class is dense in $C(K)$ for every compact set $K\subset\mathbb{R}^n$
if and only if $\sigma$ is nonpolynomial.
\end{theorem}

\begin{proof}
We first prove sufficiency.

Assume that $\sigma$ is continuous and nonpolynomial. We show that every
function representable by a standard single-hidden-layer multilayer perceptron
can be realized exactly by a KAN of the form
\eqref{eq:kan_layer1}--\eqref{eq:kan_output}.

Let
\[
G(x)=\sum_{k=1}^{N} c_k\,\sigma(w_k\cdot x+b_k),
\qquad x\in\mathbb{R}^n,
\]
be an arbitrary single-hidden-layer neural network, where
$w_k=(w_{k1},\dots,w_{kn})\in\mathbb{R}^n$ and $b_k,c_k\in\mathbb{R}$.

We show that $G$ can be written in the KAN form. For the first 
hidden layer, define affine edge functions by
\[
\psi_{kj}^{(1)}(t)=w_{kj}t+\beta_{kj},
\qquad j=1,\dots,n,\quad k=1,\dots,N,
\]
where the constants $\beta_{kj}$ are chosen so that
\[
\sum_{j=1}^{n}\beta_{kj}=b_k
\qquad \text{for each } k.
\]
Then
\[
h_k^{(1)}(x)
=
\sum_{j=1}^{n}\psi_{kj}^{(1)}(x_j)
=
\sum_{j=1}^{n}(w_{kj}x_j+\beta_{kj})
=
w_k\cdot x+b_k.
\]

For the second hidden layer, define
\[
\psi_{kr}^{(2)}(t)=
\begin{cases}
\sigma(t), & r=k,\\[1mm]
0, & r\ne k,
\end{cases}
\]
where $0$ denotes the zero function, which is affine. Then
\[
h_k^{(2)}(x)
=
\sum_{r=1}^{N}\psi_{kr}^{(2)}\bigl(h_r^{(1)}(x)\bigr)
=
\sigma\bigl(h_k^{(1)}(x)\bigr)
=
\sigma(w_k\cdot x+b_k).
\]

Finally, define the output edge functions by
\[
\phi_k(t)=c_k t,
\]
which are affine. Then
\[
F(x)
=
\sum_{k=1}^{N}\phi_k\bigl(h_k^{(2)}(x)\bigr)
=
\sum_{k=1}^{N} c_k\,h_k^{(2)}(x)
=
\sum_{k=1}^{N} c_k\,\sigma(w_k\cdot x+b_k)
=
G(x).
\]

Thus every single-hidden-layer multilayer perceptron with activation $\sigma$
is exactly representable by a two-hidden-layer KAN of the stated type.

Since $\sigma$ is continuous and nonpolynomial, the classical universal
approximation theorem (see \cite{Leshno1993} or \cite{Pinkus1999}) 
implies that the class of functions
\[
x\mapsto \sum_{k=1}^{N} c_k\,\sigma(w_k\cdot x+b_k)
\]
is dense in $C(K)$. Therefore the above KAN class is also dense in $C(K)$.

\medskip

We now prove necessity. Assume that $\sigma$ is a polynomial.

If $\sigma$ is affine, then every admissible edge function is affine. 
It follows that every function representable by
\eqref{eq:kan_layer1}--\eqref{eq:kan_output} is affine on $\mathbb{R}^n$.
Therefore this class cannot be dense in $C(K)$ for general compact sets
$K\subset\mathbb{R}^n$, for example for $K=[0,1]^n$.

Assume now that $\sigma$ is a polynomial of degree $d\ge 2$. Since each edge
function is either affine or equal to $\sigma$, a simple induction on the
layers shows that every function representable by 
\eqref{eq:kan_layer1}--\eqref{eq:kan_output} is a polynomial of degree bounded
by a constant depending only on $d$ and the number of hidden layers (here equal to $2$). 
Hence this class is contained in a finite-dimensional space, and for many 
compact sets $K\subset\mathbb{R}^n$ it is not dense in $C(K)$.

This completes the proof.
\end{proof}

Theorem~\ref{thm:kan_fixed_sigma_affine} shows that, for 
the considered two-hidden-layer KAN architecture,
nonpolynomiality of the fixed admissible function 
$\sigma$ is both necessary and sufficient
for universality. This is in direct analogy 
with the classical universal approximation theorem for shallow neural networks.

\section{Universality with a fixed polynomial edge function}

In this section, we prove that deep Kolmogorov--Arnold Networks are universal
approximators even if the only non-affine edge function available is a single
fixed nonlinear polynomial.

Let $K\subset\mathbb R^n$ be compact, and let $\mathcal E\subset C(\mathbb R)$
be a given class of univariate functions. We denote by
$\mathcal K_{\mathcal E}(K)$ the class of all functions $F:K\to\mathbb R$
representable on $K$ by deep Kolmogorov--Arnold Networks of the form
\begin{align}
h_j^{(0)}(x) &= x_j, \qquad j=1,\dots,n, \label{eq:poly_input}\\
h_k^{(\ell)}(x) &=
\sum_{j=1}^{n_{\ell-1}}
\psi_{kj}^{(\ell)}\bigl(h_j^{(\ell-1)}(x)\bigr),
\qquad \ell=1,\dots,L,\quad k=1,\dots,n_\ell, \label{eq:poly_hidden}\\
F(x) &= \sum_{k=1}^{n_L}\phi_k\bigl(h_k^{(L)}(x)\bigr), \label{eq:poly_output}
\end{align}
where each edge function $\psi_{kj}^{(\ell)}$ and $\phi_k$ belongs to
$\mathcal E$.

In the special case where $\mathcal E$ consists of all affine functions
together with a given function $\sigma$, we write $\mathcal K_\sigma(K)$
instead of $\mathcal K_{\mathcal E}(K)$. These notations will be used
throughout the paper.

\begin{theorem}\label{thm:poly_edge_uat}
Let $\sigma:\mathbb R\to\mathbb R$ be a polynomial of degree at least $2$.
Then for every compact set $K\subset\mathbb R^n$, the class
$\mathcal K_\sigma(K)$ is dense in $C(K)$ with respect to the uniform norm.
\end{theorem}

\begin{proof}
We show that $\mathcal K_\sigma(K)$ is a subalgebra of $C(K)$ that contains
the constants and the coordinate functions. Since the coordinate functions
separate points of $K$, the Stone--Weierstrass theorem will then imply that
$\mathcal K_\sigma(K)$ is dense in $C(K)$.

We begin by observing that $\mathcal K_\sigma(K)$ contains the constants and
the coordinate functions. Indeed, for each $j=1,\dots,n$, 
the coordinate function $x\mapsto x_j$
is representable by a KAN with one hidden layer, obtained by taking
the affine identity map on the $j$th input and zero affine maps on the others.
Likewise, every constant function is representable by taking constant affine
edge functions.

Next, we show that $\mathcal K_\sigma(K)$ is closed under affine combinations.
Let $f,g\in\mathcal K_\sigma(K)$ and let $a,b,c\in\mathbb R$. Then the function
\[
x\longmapsto af(x)+bg(x)+c
\]
also belongs to $\mathcal K_\sigma(K)$. Indeed, let $N_f$ and $N_g$ be 
networks representing $f$ and $g$, respectively.
If these networks have different depths, we equalize their depths by
appending additional layers after the output node, each consisting of a
single unit whose incoming edge function is the identity map. We then form 
a larger network by taking, at each layer, all hidden units of
$N_f$ together with all hidden units of $N_g$. The edge functions are chosen so
that the units corresponding to $N_f$ depend only on the preceding units of
$N_f$, while the units corresponding to $N_g$ depend only on the preceding units
of $N_g$ (that is, for each layer, the edge functions from units of $N_f$ to units
of $N_g$, and from units of $N_g$ to units of $N_f$, are taken to be identically
zero). In this way, the two networks $N_f$ and $N_g$ are arranged in parallel within a
single KAN. Consequently, there exists a layer of this network that contains
two units whose values are $f(x)$ and $g(x)$, respectively. We then add one additional 
layer consisting of a single unit that computes $af(x)+bg(x)$ via affine edge functions 
$t\mapsto at$ and $t\mapsto bt$, followed by the output edge function $t\mapsto t+c$ that 
adds the constant term $c$. Thus the function $x\mapsto af(x)+bg(x)+c$ belongs to
$\mathcal K_\sigma(K)$.

We now verify closure under composition with scalar-input networks. Suppose
$u\in\mathcal K_\sigma(K)$ and let $G:\mathbb R\to\mathbb R$
be representable by a network of the form
\eqref{eq:poly_input}--\eqref{eq:poly_output} with the same edge
functions. Let $N_u$ be a network representing $u$, and let $N_G$
be a network with scalar input representing $G$. We obtain a network
representing $G\circ u$ by feeding the output of $N_u$ as the input to $N_G$.
All edge functions in the resulting network are still either affine or equal
to $\sigma$. Hence $G\circ u\in\mathcal K_\sigma(K)$.

We next show that the squaring map $S(t)=t^2$ is exactly representable by a
scalar-input deep KAN whose edge functions are affine or equal to $\sigma$.
Writing
\[
\sigma(t)=a_d t^d+a_{d-1}t^{d-1}+\cdots+a_1 t+a_0,
\qquad a_d\ne 0,\quad d:=\deg \sigma\ge 2,
\]
and fixing $h\ne 0$, define the difference operator
\[
(\Delta_h f)(t):=f(t+h)-f(t),
\qquad
\Delta_h^m f:=\Delta_h(\Delta_h^{m-1}f).
\]
It is well known that $\Delta_h^m \sigma$ is a polynomial of degree $d-m$
with leading coefficient $a_d \frac{d!}{(d-m)!} h^m$. In particular, for
$m=d-2$ the function
\[
q(t):=\Delta_h^{d-2}\sigma(t)
\]
is a quadratic polynomial,
\[
q(t)=A t^2+Bt+C,
\qquad
A=a_d\frac{d!}{2}h^{d-2}\ne 0.
\]

On the other hand,
\[
q(t)=\sum_{r=0}^{d-2} c_r\,\sigma(t+r h),
\qquad
c_r:=(-1)^{d-2-r}\binom{d-2}{r}.
\]
We now construct a two-hidden-layer KAN representing $q$. The first hidden
layer computes the affine functions $u_r(t)=t+r h$. The second hidden layer
consists of nodes $v_r$ defined by
\[
v_r(t)=\sigma(u_r(t))=\sigma(t+r h),
\]
which is achieved by taking the incoming edge from $u_r$ to be $\sigma$ and
all other incoming edges to be zero. The output node then computes
\[
\sum_{r=0}^{d-2} c_r\,v_r(t)=q(t)
\]
using affine edge functions $x\mapsto c_r x$. Hence $q$ is exactly representable.

Since $q(t)=A t^2+Bt+C$ with $A\ne 0$, we have
\[
t^2=A^{-1}q(t)-A^{-1}Bt-A^{-1}C.
\]
Since $\mathcal K_\sigma(K)$ is closed under affine combinations and
both $q(t)$ and $t$ are representable, it follows that $t\mapsto t^2$
is also representable. That is, $t^2\in \mathcal K_\sigma(K)$.

We now establish closure under multiplication. Let $f,g\in\mathcal K_\sigma(K)$.
Since $\mathcal K_\sigma(K)$ is closed under affine combinations, the function
$f+g$ belongs to $\mathcal K_\sigma(K)$. Since the squaring map $S(t)=t^2$ is
exactly representable by a scalar-input network with edge functions affine
or $\sigma$, it follows from the composition property established above that
$S\circ f$, $S\circ g$, and $S\circ (f+g)$ all belong to
$\mathcal K_\sigma(K)$. Using the polarization identity
\[
fg
=
\frac{(f+g)^2-f^2-g^2}{2},
\]
we conclude that $fg\in\mathcal K_\sigma(K)$. Thus $\mathcal K_\sigma(K)$ is
closed under pointwise multiplication.

It follows that $\mathcal K_\sigma(K)$ is a subalgebra of $C(K)$ that contains
the constants and the coordinate functions. Therefore the restriction to $K$
of every polynomial in the variables $x_1,\dots,x_n$ belongs to
$\mathcal K_\sigma(K)$.

We see that $\mathcal K_\sigma(K)$ is a subalgebra of $C(K)$ containing the
constants and separating points of $K$. By the Stone--Weierstrass theorem,
\[
\overline{\mathcal K_\sigma(K)}=C(K).
\]
This completes the proof.
\end{proof}

It should be emphasized that the role of depth is essential in this theorem.
A shallow KAN in which all non-affine edge functions are given by a fixed
polynomial function $\sigma$ can generate only multivariate polynomials
of uniformly bounded degree, whereas deeper architectures allow the
exact realization of the squaring map and hence multiplication. Consequently,
one can construct arbitrary monomials and thus all multivariate polynomials.

\section{Universality with a fixed non-affine edge function}

In the previous sections, we established universality of
Kolmogorov--Arnold Networks for continuous nonpolynomial edge functions
(Theorem~\ref{thm:kan_fixed_sigma_affine}) and for polynomial edge
functions of degree at least~$2$ (Theorem~\ref{thm:poly_edge_uat}). We
now combine these results to obtain a necessary and sufficient condition.

Recall that $\mathcal K_\sigma(K)$ denotes the class of functions on $K$
representable by deep Kolmo\-gorov--Arnold Networks with edge functions
either affine or equal to $\sigma$.

\begin{theorem}\label{thm:kan_iff}
Let $\sigma:\mathbb R\to\mathbb R$ be continuous. Then
$\mathcal K_\sigma(K)$ is dense in $C(K)$ for every compact set
$K\subset\mathbb R^n$ if and only if $\sigma$ is non-affine.
\end{theorem}

\begin{proof}
If $\sigma$ is non-affine and nonpolynomial, the conclusion follows from
Theorem~\ref{thm:kan_fixed_sigma_affine}. If $\sigma$ is a polynomial of
degree at least $2$, the conclusion follows from
Theorem~\ref{thm:poly_edge_uat}.

Conversely, if $\sigma$ is affine, then all admissible edge functions are affine, 
and hence every network in $\mathcal K_\sigma(K)$ is
affine. Therefore $\mathcal K_\sigma(K)$ cannot be dense in $C(K)$ for every
compact set $K\subset\mathbb R^n$; for example, it is not dense in
$C([0,1]^n)$.
\end{proof}

Thus, Theorem~\ref{thm:kan_iff} shows that universality holds if and only if
the edge functions are not all affine. That is, the purely affine case is
the only obstruction.

\medskip

We now apply Theorem~\ref{thm:kan_iff} to the class of edge 
functions used in the original work of Liu et al.~\cite{Liu2024}. 
In that paper, Kolmogorov--Arnold Networks 
are introduced with a flexible class of admissible edge functions. 
In the concrete implementation proposed in \cite{Liu2024}, 
these functions are parameterized as combinations 
of a fixed base function and a spline component. 
More precisely, each edge function is taken in the form
\begin{equation}\label{eq:kan_edge_original}
\phi(t)
=
w_b\,b(t)
+
w_s\,s(t),
\qquad w_b,w_s\in\mathbb R,
\end{equation}
where $b:\mathbb R\to\mathbb R$ is a fixed continuous base function (chosen as the 
SiLU function in the numerical implementation), and $s$ belongs to a spline space 
of a given degree associated with a knot sequence.

We briefly recall the structure of such spline spaces 
(see, e.g., \cite{deBoor2001,Schumaker2007}). Let $k\ge 1$ and let
\[
\tau=\{t_0<t_1<\cdots<t_M\}
\]
be a finite sequence of knots in $\mathbb R$. The associated spline space 
$S_{k,\tau}$ consists of all functions $s:\mathbb R\to\mathbb R$ that are 
polynomials of degree at most $k$ on each interval $[t_i,t_{i+1}]$, and 
that are globally $C^{k-1}$-smooth. Equivalently, every such function 
admits a representation
\[
s(t)=\sum_{i=1}^N c_i B_i(t),
\]
where $\{B_i\}$ is a B-spline basis corresponding to the 
knot sequence $\tau$ (see \cite{deBoor2001}). 
It is clear that spline spaces of degree $k\ge 1$ contain 
all affine functions, and, provided the knot sequence has at 
least one interior knot, also contain functions that are not affine.

We emphasize that the theoretical results in \cite{Liu2024} 
do not establish a universal approximation theorem 
for Kolmogorov--Arnold Networks in the classical sense. 
Instead, the approximation result (see Theorem~2.1 in \cite{Liu2024}) 
shows that, under the assumption that a target function 
admits a KAN representation with sufficiently smooth 
univariate components, spline-parameterized KANs approximate this 
representation with an error that decays as the spline grid size increases. 
More precisely, the approximation error is controlled 
in terms of the number of grid intervals 
(or, equivalently, the number of spline knots), and 
improving the approximation requires increasing this number.

In particular, this result does not imply that the introduced 
KANs are dense in the space of all continuous functions, 
but rather provides an approximation guarantee within 
a prescribed representation class. From the approximation-theoretic 
point of view, however, it is essential to determine whether 
KANs with edge functions of the form \eqref{eq:kan_edge_original} 
are universal approximators independently of 
any prior representation of the target function.

The following corollary shows that this is indeed the case, even when the 
degree of the splines and the number of knots are fixed in advance.

\begin{corollary}\label{cor:kan_original_edge}
Let $b:\mathbb R\to\mathbb R$ be a continuous function, and let $k\ge 1$ 
and $\tau$ be a fixed knot sequence. Let $\mathcal E$ 
denote the class of functions of the form
\[
\phi(t)=w_b\,b(t)+w_s\,s(t),
\qquad w_b,w_s\in\mathbb R,\quad s\in S_{k,\tau}.
\]
Then, for every compact set $K\subset\mathbb R^n$, the class of deep 
Kolmogorov--Arnold Networks whose edge functions 
belong to $\mathcal E$ is dense in $C(K)$.
\end{corollary}

\begin{proof}
Since $S_{k,\tau}$ contains all affine functions, the class $\mathcal E$ 
also contains all affine functions (by taking $w_b=0$). On the other hand, 
$S_{k,\tau}$ contains non-affine functions. Hence the admissible edge class 
$\mathcal E$ contains all affine functions together with a non-affine function. 
The conclusion now follows directly from Theorem~\ref{thm:kan_iff}.
\end{proof}

This result shows that, unlike the approximation approach of \cite{Liu2024}, 
universality of Kolmogorov--Arnold Networks does not rely on increasing the number 
of spline knots. In particular, both the degree of the splines and the number of 
knot points can be fixed a priori. Thus, expressive power is achieved through 
depth rather than by enlarging the class of univariate edge functions.

\section{Universality with finite families of affine functions}

In Theorem~\ref{thm:kan_iff}, universality was established under the assumption
that all affine functions are admissible. A natural question is whether the
full affine family can be replaced by a finite one.

We first show that, in the nonpolynomial case, a fixed family of five affine
functions suffices. We then combine this with the polynomial case to obtain a
general finite-affine universality theorem.

Recall that $\mathcal K_{\mathcal E}(K)$ denotes the class of functions
$F:K\to\mathbb R$ representable by deep Kolmogorov--Arnold Networks of the form
\eqref{eq:poly_input}--\eqref{eq:poly_output}, where the edge functions
$\psi_{kj}^{(\ell)}$ and $\phi_k$ belong to $\mathcal E$.

\begin{theorem}\label{thm:nonpoly_five_affine}
Let $\sigma:\mathbb R\to\mathbb R$ be continuous and nonpolynomial, and let
\[
A_0=
\Bigl\{
t\mapsto 0,\quad
t\mapsto 1,\quad
t\mapsto t,\quad
t\mapsto -t,\quad
t\mapsto t/2
\Bigr\}.
\]
Then, for every compact set $K\subset\mathbb R^n$, the class
$\mathcal K_{A_0\cup\{\sigma\}}(K)$ is dense in $C(K)$.
\end{theorem}

\begin{proof}
Let
\[
\mathbb D=\left\{\frac{m}{2^r}: m\in\mathbb Z,\ r\in\mathbb N\cup\{0\}\right\}
\]
denote the set of dyadic rational numbers.

We first claim that every affine map
\[
t\mapsto qt+b,
\qquad q,b\in\mathbb D,
\]
is exactly representable by a scalar-input deep KAN whose edge
functions belong to $A_0$.

We begin by constructing dyadic linear maps. Consider a scalar-input KAN
with $r$ hidden layers, each consisting of a single node, and with edge
function $t\mapsto t/2$ at every layer. Then the network computes
\[
t\mapsto \frac{t}{2^r}.
\]
Denote this map by $u(t)=t/2^r$.

We next construct
\[
t\mapsto \frac{m}{2^r}t.
\]
If $m\ge 0$, we create $m$ copies of $u(t)$ using a layer whose incoming
edge functions are all the identity map, and then sum these copies at a
single node, again using identity edge functions. This yields
\[
m\,u(t)=\frac{m}{2^r}t.
\]
If $m<0$, we first construct $(|m|/2^r)t$ as above and then apply the edge
function $t\mapsto -t$. Thus, for every $m\in\mathbb Z$ and $r\ge 0$, the map
\[
t\mapsto \frac{m}{2^r}t
\]
is representable.

We now turn to dyadic constant maps, that is,
\[
t\mapsto b,
\qquad b\in\mathbb D.
\]

First consider integer constants. Fix $m\in\mathbb Z$. If $m\ge 0$, we take
one hidden layer consisting of $m$ nodes, each receiving the input through
the edge function $t\mapsto 1$. Thus each node computes the constant function
$1$. We then add one further node whose incoming edges from these $m$ nodes
are all the identity map, so that this node computes
\[
1+\cdots+1 = m.
\]
If $m<0$, we first construct the constant map $t\mapsto |m|$ and then add one
more layer with a single node whose incoming edge function is $t\mapsto -t$.
This yields the constant map $t\mapsto m$.

Next, for any $r\ge 0$, we apply $r$ successive layers, each consisting of a
single node with incoming edge function $t\mapsto t/2$. This transforms the
constant map $t\mapsto m$ into
\[
t\mapsto \frac{m}{2^r}.
\]
Thus every dyadic constant map is representable.

Finally, we combine the two constructions. We have shown that for any
$q\in\mathbb D$ the map $t\mapsto qt$ is representable, and for any
$b\in\mathbb D$ the map $t\mapsto b$ is representable. To construct
\[
t\mapsto qt+b,
\]
we place the corresponding subnetworks in parallel so that their outputs are
available at the same layer (for this procedure, see the proof of 
Theorem~\ref{thm:poly_edge_uat}), and then add one further node whose incoming
edges from these outputs are both the identity map. This node computes
\[
t\mapsto qt + b.
\]

Thus every dyadic affine map is exactly representable by a scalar-input deep
KAN whose edge functions belong to $A_0$.

We next show that every single-hidden-layer neural network with dyadic parameters
can be exactly realized by a deep KAN whose edge functions belong to
$A_0\cup\{\sigma\}$. Consider
\begin{equation}\label{eq:dyadic_shallow}
G(x)=\sum_{k=1}^N c_k\,\sigma(w_k\cdot x+b_k),
\end{equation}
where
\[
w_k=(w_{k1},\dots,w_{kn})\in\mathbb D^n,
\qquad
b_k,c_k\in\mathbb D.
\]

Fix $k\in\{1,\dots,N\}$. We first construct the map
\[
x\mapsto w_k\cdot x+b_k
\]
using a KAN with edge functions in $A_0$.

For each $j=1,\dots,n$, the map
\[
x_j \mapsto w_{kj}x_j
\]
is representable by a scalar-input KAN with edge functions in $A_0$,
since $w_{kj}\in\mathbb D$. Likewise, the constant map
\[
x\mapsto b_k
\]
is representable using only functions from $A_0$.

We place these subnetworks in parallel so that the outputs
\[
w_{k1}x_1,\ \dots,\ w_{kn}x_n,\ b_k
\]
are available at the same layer. We then add one further node whose incoming
edges from these outputs are all the identity map. This node computes
\[
x\mapsto \sum_{j=1}^n w_{kj}x_j + b_k
=
w_k\cdot x + b_k.
\]

Next, we apply the function $\sigma$ by introducing one additional layer
consisting of a single node whose incoming edge from this quantity is
$t\mapsto \sigma(t)$. This yields
\[
x\mapsto \sigma(w_k\cdot x + b_k).
\]

We now construct
\[
x\mapsto c_k\,\sigma(w_k\cdot x + b_k).
\]
Since $c_k\in\mathbb D$, write $c_k = m/2^r$ with $m\in\mathbb Z$ and $r\ge 0$.
Starting from the function
\[
x\mapsto \sigma(w_k\cdot x + b_k),
\]
we apply $r$ successive layers with edge function $t\mapsto t/2$ to obtain
\[
x\mapsto \frac{1}{2^r}\sigma(w_k\cdot x + b_k).
\]
If $m\ge 0$, we construct $m$ copies of this function and sum them at a single
node using identity edge functions, which yields
\[
x\mapsto \frac{m}{2^r}\sigma(w_k\cdot x + b_k).
\]
If $m<0$, we first construct
\[
x\mapsto \frac{|m|}{2^r}\sigma(w_k\cdot x + b_k)
\]
in the same way, and then apply the edge function $t\mapsto -t$.
Thus the function $x\mapsto c_k\,\sigma(w_k\cdot x + b_k)$ is representable.

Finally, we combine the $N$ subnetworks corresponding to $k=1,\dots,N$. We
place them in parallel so that the outputs
\[
c_1\,\sigma(w_1\cdot x + b_1),\ \dots,\
c_N\,\sigma(w_N\cdot x + b_N)
\]
are available at the same layer, and then add one further node whose incoming
edges from these outputs are all the identity map. This node computes
\[
x\mapsto \sum_{k=1}^N c_k\,\sigma(w_k\cdot x + b_k),
\]
which coincides with \eqref{eq:dyadic_shallow}.

Thus every function of the form \eqref{eq:dyadic_shallow} is exactly
representable by a deep KAN whose edge functions belong to
$A_0\cup\{\sigma\}$.

We now prove density in $C(K)$. Let $f\in C(K)$ and let $\varepsilon>0$ be
given. Since $\sigma$ is continuous and nonpolynomial, it follows from the
classical universal approximation theorem \cite{Leshno1993} that there exist
$N\in\mathbb N$, vectors $w_k\in\mathbb R^n$, and scalars
$b_k,c_k\in\mathbb R$ such that
\[
H(x)=\sum_{k=1}^N c_k\,\sigma(w_k\cdot x+b_k)
\]
satisfies
\[
\sup_{x\in K}|f(x)-H(x)|<\frac{\varepsilon}{2}.
\]

The function $H(x)$ depends continuously on the parameters $w_k$, $b_k$, 
and $c_k$, uniformly on $K$. Indeed, since $K$ is compact, 
there exists $R>0$ such that
\[
\|x\|\le R \qquad \text{for all } x\in K.
\]
For fixed $k$, if $w_k'\to w_k$ and $b_k'\to b_k$, then for every $x\in K$,
\[
\bigl|(w_k'\cdot x+b_k')-(w_k\cdot x+b_k)\bigr|
\le R\|w_k'-w_k\| + |b_k'-b_k|.
\]
Hence
\[
\sup_{x\in K}\bigl|(w_k'\cdot x+b_k')-(w_k\cdot x+b_k)\bigr|\to 0.
\]

Since $w_k'\to w_k$ and $b_k'\to b_k$, the quantities
\[
w_k'\cdot x+b_k' \quad \text{and} \quad w_k\cdot x+b_k
\]
remain in a fixed compact interval $I\subset\mathbb R$ for all $x\in K$ and all
parameters sufficiently close to $(w_k,b_k)$. Since $\sigma$ is continuous, it
is uniformly continuous on $I$. Therefore
\[
\sup_{x\in K}
\left|
\sigma(w_k'\cdot x+b_k')-\sigma(w_k\cdot x+b_k)
\right|
\to 0.
\]

It follows that, if in addition $c_k'\to c_k$, then
\[
\sup_{x\in K}
\left|
c_k'\sigma(w_k'\cdot x+b_k')-c_k\sigma(w_k\cdot x+b_k)
\right|
\to 0.
\]
Summing over $k=1,\dots,N$, we conclude that the map
\[
(w_k,b_k,c_k)_{k=1}^N
\mapsto
\sum_{k=1}^N c_k\,\sigma(w_k\cdot x+b_k)
\]
depends continuously on the parameters in the uniform norm on $K$.

Since the dyadic rationals are dense in $\mathbb R$, we may choose dyadic
approximations
\[
w_k'\in\mathbb D^n,\qquad b_k',c_k'\in\mathbb D
\]
such that the corresponding dyadic-parameter network
\[
H'(x)=\sum_{k=1}^N c_k'\,\sigma(w_k'\cdot x+b_k')
\]
satisfies
\[
\sup_{x\in K}|H(x)-H'(x)|<\frac{\varepsilon}{2}.
\]

By the previous part of the proof, $H'$ is exactly representable by a deep KAN
whose edge functions belong to $A_0\cup\{\sigma\}$. Hence
\[
\sup_{x\in K}|f(x)-H'(x)|
\le
\sup_{x\in K}|f(x)-H(x)|
+
\sup_{x\in K}|H(x)-H'(x)|
<
\varepsilon.
\]

Since $f$ and $\varepsilon$ were arbitrary, the KAN class generated by
$A_0\cup\{\sigma\}$ is dense in $C(K)$.
\end{proof}

Note that in the above theorem the role of the family $A_0$ is 
to generate, by composition and summation,
all affine maps with dyadic coefficients within the Kolmogorov--Arnold 
network architecture. This makes it possible to realize, inside this class of 
networks, all shallow $\sigma$-networks with dyadic parameters exactly. 
In combination with the classical universal approximation theorem and 
the density of dyadic numbers, this explains 
the resulting universality on compact sets.

Each function in the family
\[
A_0=\{0,1,t,-t,t/2\}
\]
has a specific role in the construction.
The zero function suppresses unwanted connections between layers.
The constant function $t\mapsto1$ is used to construct constant maps.
The identity map $t\mapsto t$ preserves and copies quantities unchanged.
The map $t\mapsto -t$ produces sign changes, while the map
$t\mapsto t/2$ generates scalings by 
$2^{-r}$. Combined with the summation structure of the network,
this yields affine maps whose coefficients are dyadic rationals.

\medskip

\textbf{Remark.} The use of dyadic rationals in the proof is not essential. They arise naturally 
from the choice of the affine family
\[
A_0=\{0,1,t,-t,t/2\}.
\]

More generally, one may replace the map $t\mapsto t/2$ by $t\mapsto t/q$, 
where $q\ge2$ is a fixed integer. In this case the proof works with the set
\[
\mathbb D_q
=
\left\{
\frac{m}{q^r}:m\in\mathbb Z,\ r\in\mathbb N\cup\{0\}
\right\},
\]
which is again dense in $\mathbb R$, and with the affine family
\[
\{0,1,t,-t,t/q\}.
\]

\medskip

The following theorem provides a characterization of universality for 
Kolmogorov--Arnold networks, showing that non-affinity of the distinguished 
edge function $\sigma$ is both necessary and sufficient when combined 
with a suitable finite family of affine edge functions.

\begin{theorem}\label{thm:finite_affine_nonaffine}
Let $\sigma:\mathbb R\to\mathbb R$ be continuous. Then the following are
equivalent:
\begin{itemize}
\item[(i)] $\sigma$ is non-affine;
\item[(ii)] there exists a finite family $A_\sigma$ of affine functions on
$\mathbb R$ such that, for every compact set $K\subset\mathbb R^n$, the class
$\mathcal K_{A_\sigma\cup\{\sigma\}}(K)$ is dense in $C(K)$.
\end{itemize}
\end{theorem}

\begin{proof}
Assume first that \((ii)\) holds. If $\sigma$ were affine, then every function
in $A_\sigma\cup\{\sigma\}$ would be affine. It would follow that every network
in the corresponding class is affine, which is impossible if the class is dense
in $C(K)$ for every compact set $K\subset\mathbb R^n$. Thus $\sigma$ must be
non-affine. This proves \((ii)\Rightarrow(i)\).

We now prove \((i)\Rightarrow(ii)\). Assume that $\sigma$ is continuous and
non-affine.

If $\sigma$ is nonpolynomial, then the conclusion follows from
Theorem~\ref{thm:nonpoly_five_affine} by taking
\[
A_\sigma=A_0.
\]

If $\sigma$ is a polynomial of degree at least $2$, then by the proof of
Theorem~\ref{thm:poly_edge_uat}, the squaring map $t\mapsto t^2$ can be
constructed from $\sigma$ using affine maps from a finite family. These consist
of translations and scalar multiples arising in the finite-difference
representation of $\Delta_h^{d-2}\sigma$, together with the affine maps used in
the final reconstruction of $t^2$ from the quadratic polynomial
$q(t)=\Delta_h^{d-2}\sigma(t)$. Let $B_\sigma$ denote this finite
family, and set
\[
A_\sigma := A_0 \cup B_\sigma,
\]
where
\[
A_0=
\Bigl\{
t\mapsto 0,\quad
t\mapsto 1,\quad
t\mapsto t,\quad
t\mapsto -t,\quad
t\mapsto t/2
\Bigr\}.
\]

Since $A_0\subset A_\sigma$, the dyadic affine maps are exactly representable
by scalar-input deep KANs with edge functions in $A_\sigma$ 
(see the proof of Theorem~\ref{thm:nonpoly_five_affine}). Note that this 
construction does not depend on any specific properties of $\sigma$.
On the other hand, the proof of Theorem~\ref{thm:poly_edge_uat} shows that,
using only functions from $B_\sigma\cup\{\sigma\}\subset A_\sigma\cup\{\sigma\}$,
one can exactly realize the squaring map. Hence, using only functions from
$A_\sigma\cup\{\sigma\}$, one can realize multiplication via the polarization
identity
\[
fg=\frac{(f+g)^2-f^2-g^2}{2},
\]
and therefore construct exactly all multivariate polynomials with dyadic
coefficients.

Finally, polynomials with dyadic coefficients are dense in $C(K)$ for every
compact set $K\subset\mathbb R^n$, since ordinary polynomials are dense by the
Stone--Weierstrass theorem and real coefficients may be approximated
arbitrarily well by dyadic numbers. Therefore the resulting network class is
dense in $C(K)$.

Thus \((ii)\) holds in both cases, and the proof is complete.
\end{proof}

Theorem~\ref{thm:finite_affine_nonaffine} shows that the full affine family in
Theorem~\ref{thm:kan_iff} can always be reduced to a finite one. In the
nonpolynomial case, the explicit five-element family $A_0$ already suffices.

Note that in Theorem~\ref{thm:finite_affine_nonaffine} the finite family of 
affine functions $A_\sigma$ depends on $\sigma$. It remains an interesting open 
problem to determine a minimal finite affine
family $A$, independent of $\sigma$, with the following property. For every
continuous non-affine function $\sigma$ and every compact set
$K\subset\mathbb R^n$, the class of Kolmogorov--Arnold networks with
edge functions in $A\cup\{\sigma\}$ is dense in $C(K)$. It is also natural to
ask whether the number five, the number of affine functions in the nonpolynomial case 
(Theorem~\ref{thm:nonpoly_five_affine}), can be reduced.

\end{document}